\documentclass{article}

\usepackage{times}
\usepackage{graphicx} 
\usepackage{subfigure}

\usepackage{natbib}

\usepackage{algorithm}
\usepackage{algorithmic}

\usepackage{hyperref}

\usepackage[accepted]{icml2014}

\RequirePackage{amsthm,amsmath,amssymb}

\newtheorem{lemma}{Lemma}
\newtheorem{theorem}{Theorem}

\newtheorem{definition}{Definition}

\newtheorem{assumption}{Assumption}

\theoremstyle{definition}

\newcommand{\D}{\mathcal{U}}

\newcommand{\real}{\mathbb{R}}
\newcommand{\expectation}{\mathbb{E}}

\newcommand{\expec}{{\Expectation} \,}

\newcommand{\norm}[1]{\left\|#1\right\|}

\newcommand{\abs}[1]{\left|#1\right|}
\newcommand{\paren}[1]{\left(#1\right)}
\newcommand{\pr}[1]{\mathbb{P}\left(#1\right)}

\newcommand{\ft}{\tilde{f}}

\newcommand{\ind}[1]{\mathbf{1}_{\left\{#1\right\}}}
\newcommand{\braces}[1]{\left\{#1\right\}}

\DeclareMathOperator*{\Expectation}{\expectation}

\def\ci{\perp\!\!\!\!\perp}

\icmltitlerunning{Consistency of Causal Inference under the Additive Noise Model}

\begin{document}

\twocolumn[
\icmltitle{Consistency of Causal Inference under the Additive Noise Model}

\icmlauthor{Samory Kpotufe}{samory@ttic.edu}
\icmladdress{Toyota Technological Institute-Chicago \\
\&  Max Planck Institute for Intelligent Systems}
\icmlauthor{Eleni Sgouritsa}{eleni.sgouritsa@tuebingen.mpg.de}
\icmladdress{Max Planck Institute for Intelligent Systems}
\icmlauthor{Dominik Janzig}{dominik.janzig@tuebingen.mpg.de}
\icmladdress{Max Planck Institute for Intelligent Systems}
\icmlauthor{Bernhard Sch\"{o}lkopf}{bs@tuebingen.mpg.de}
\icmladdress{Max Planck Institute for Intelligent Systems}

\icmlkeywords{boring formatting information, machine learning, ICML}

\vskip 0.3in
]

\begin{abstract}
We analyze a family of methods for statistical causal inference from sample under the so-called \emph{Additive Noise Model}.
While most work on the subject has concentrated on establishing the soundness of the Additive Noise Model,
the statistical consistency of the resulting inference methods has received little attention. We derive general conditions under which the
given family of inference methods consistently infers the causal direction in a nonparametric setting.
\end{abstract}
\section{Introduction}



Drawing causal conclusions for a set of observed variables given a sample from their joint distribution is a fundamental problem in science.
Conditional-independence-based methods \cite{pearl2000causality, spirtes2000causation} estimate a set of directed acyclic graphs, all entailing the same conditional independences, from the data. However, these methods can not distinguish between two graphs that entail the same set of conditional independences, the so-called \emph{Markov equivalent} graphs. Consider for example the case of only two observed dependent random variables. Conditional-independence-based methods can not recover the causal graph since $X \rightarrow Y$ and $Y \rightarrow X$ are Markov equivalent.
An elegant basis for causal graphs is the framework of structural causal models (SCMs) \cite{pearl2000causality}, where every observable is a function of its parents and an unobserved independent noise term. This allows us to formulate an assumption on function classes which lets us infer the causal direction in two-variable case.

A special case of SCMs is the \emph{Causal Additive Noise Model} (CAM) \cite{shimizu2006linear, hoyer2009nonlinear, tillman2009nonlinear, peters2011causal, peters2011identifiability} which is given as follows:
given two random variables $X$ and $Y$, $X$ is assumed to cause $Y$ if 
 ({\rm{i}}) $Y$ can be obtained as a function of $X$ plus a noise term independent of $X$, but ({\rm{ii}}) $X$ cannot be obtained as a function of $Y$ plus independent noise, then we infer that $X$ causes $Y$.
In this case, where ({\rm{i}}) and ({\rm{ii}}) hold simultaneously, the CAM is termed \emph{identifiable}.


Initial work on the CAM focused on establishing its theoretical soundness, i.e. understanding the class of distributions
$P_{X, Y}$ for which the CAM is identifiable, i.e. for which
({\rm{i}}) and ({\rm{ii}}) hold simultaneously.
Early work
by \cite{shimizu2006linear} showed that the CAM is identifiable when the functional relationship $Y=f(X) + \eta$ is
linear, provided the independent noise $\eta$ is not Gaussian. Later, \citet{hoyer2009nonlinear}, \citet{zhang2009identifiability} and \citet{peters2011causal}
showed that the CAM is identifiable more generally even if $f$ is nonlinear,
the main technical requirements being that the
marginals $P_X$, and $P_\eta$ are absolutely continuous on $\real$, with $P_\eta$ having support $\real$.
Note that \citet{zhang2009identifiability} also introduces a generalization of the CAM termed {post-nonlinear models}.
Further work by \citet{peters2011identifiability} showed how to reduce causal inference for a network of multiple variables
under the CAM to the case of two variables $X$ and $Y$ discussed so far, by properly extending the conditions
({\rm{i}}) and ({\rm{ii}}) to conditional distributions instead of marginals.
Thus, the soundness of the CAM being established by these various works, the next natural question
is to understand the statistical behavior of the resulting estimation procedures on finite samples.

Current insights into this last question are mostly empirical. Various works \cite{shimizu2006linear, hoyer2009nonlinear, peters2011causal}
have successfully validated procedures based on the CAM (outlined in Section \ref{sec:inferencemethods} below) on a mix of
artificial and real-world datasets where the causal structure to be inferred is clear.
However, on the theoretical side, 
it remains unclear
whether these procedures
can infer causality from samples in general situations where the CAM is identifiable. In the particular case where the functional
relation between $X$ and $Y$ is linear, \citet{hyvarinen2008causalConsistency} proposed a successful 
method shown to be consistent. In a recent Arxived result appearing after our initial submission, 
\citet{cam2013} showed the consistency of a maximum log-likelihood approach to causal inference under the 
multi-variable network extension of \citet{peters2011identifiability}. 

While consistency has been shown for particular procedures, in this paper we are rather interested in 
general conditions under which common approaches, with various algorithmic instantiations, are consistent. 
We derive both algorithmic and distributional conditions for statistical consistency in general situations
where the CAM is identifiable. The present work focuses on the case of two real variables, allowing us to focus on the 
inherent difficulties of achieving consistency with the common algorithmic approaches. These 
difficulties, described in Section \ref{sec:difficulties} have to do with estimating the \emph{degree} of 
independence between noise and input, while the noise is itself estimated from the input and hence is inherently dependent 
on the input.

\subsection{Inference Methods Under the Additive Noise Model}
\label{sec:inferencemethods}
Causal inference methods under the Additive Noise Model typically follow the meta-procedure below. Assume
$f$ and $g$ are the best functional fits under some risk, respectively $Y\approx f(X)$ and $X\approx g(Y)$:
\begin{quote}
 Fit $Y$ as a function $f(X)$, obtain the residuals $\eta_{Y, f} = Y-f(X)$, fit $X$ as a function $g(Y)$, obtain the residuals
 $\eta_{X, g} = X-g(Y)$, decide $X\to Y$ if $\eta_{Y,f} \ci X$ but $\eta_{X, g} \not\ci Y$, decide $Y\to X$ if the reverse holds true, abstain
 otherwise.
\end{quote}
Instantiations thus vary in the regression procedures employed for function fitting, and in the independence measures employed.
Our analysis concerns procedures employing an entropy-based independence measure, which is cheaper 
than usual independence tests. 
 These procedures vary in the regression
and entropy estimators employed. They are presented in detail in Section \ref{sec:algorithm0}.

\subsection{Towards Consistency: Main Difficulties}
\label{sec:difficulties}
Assume ({\rm{i}}) and ({\rm{ii}}) hold so that $X$ \emph{causes} $Y$ under the CAM.
We want to detect this from sufficiently large finite samples. This is consistency in a rough sense.

Establishing consistency of the above meta-procedure faces many subtle difficulties.
The above outlined algorithmic approach consists of four interdependent statistical estimation tasks,
namely two regression problems and two independence-tests. Considered separately, the consistency
of such estimation tasks is well understood, but in the present context the success of the independence 
tests is contingent on successful regression. 

The main difficulty is that although we are observing $X$ and $Y$,
we are not observing the residuals $\eta_{Y, f}$ and $\eta_{X, g}$, but empirical approximations $\eta_{Y, f_n}$ and
$\eta_{X, g_n}$ obtained by estimating $f$ and $g$ as $f_n$ and $g_n$ on a sample of size $n$.

For now, consider just detecting that $\eta_{Y, f}$, $f$ unknown, is independent from $X$.
A good estimator $f_n$ will ensure that $f_n$ and $f$ are \emph{close}, usually
in an $L_2$ sense (i.e. $\expec_X\abs{f_n(X) - f(X)}^2\approx 0$). Hence $\eta_{Y, f_n}$ is \emph{close} to $\eta_{Y, f}$,
but unfortunately this does not imply that $\eta_{Y, f_n}\ci X$ if $\eta_{Y, f}\ci X$. In fact it is
easy to construct r.v.'s $A, B, C$ such that $A\ci B$, $\abs{B -C}<\epsilon$, for arbitrary $\epsilon$, but $C\not \ci A$.
Thus, the estimate $\eta_{Y, f_n}$ might be \emph{close} to $\eta_{Y, f}$, yet it might still appear dependent on $X$ even if
$\eta_{Y, f}$ is not. Complicating matters further, $\eta_{Y, f_n}$ and $\eta_{Y, f}$ would only be close in an
average sense (instead of close for every value of $X$) since $f_n$ and $f$ are typically only close in an average sense (e.g. close in $L_2$).

Now consider the full causal discovery, i.e. consider also detecting that $\eta_{X, g}$ depends on $Y$.
To achieve consistency, the independence test employed must detect more
dependence between $\eta_{X, g_n}$ and $Y$ than between $\eta_{Y, f_n}$ and $X$.
This will depend on how the particular independence test is influenced by errors in
the particular regression procedures employed, and the relative rates at which these various procedures converge.

As previously mentioned, we will consider a family of independence-tests based on comparing sums of entropies.
We will handle the above difficulties and derive conditions for consistency
by first understanding how the various estimated entropies converge as a function
of regression convergence ($L_2$ convergence).

We do not consider the question of finite-sample convergence rates for causal estimation under the CAM.
In fact, it is not even clear whether it is generally possible to establish such rates. This is
because it is generally possible that the Bayes best fits $f(x)=\expec[Y|x]$ is smooth while $g(y)=\expec[X|y]$ is not
even continuous; yet it is well known that without smoothness or similar structural conditions,
arbitrarily bad rates of convergence are possible in regression (see e.g. \cite{GKKW:81}, Theorem 3.1).

However, along the way of deriving consistency, we analyze
the convergence of various quantities, which 
appear to affect the finite-sample behavior of the meta-procedure. 
In particular the tails of the additive noise and the richness of the regression algorithms seem to have a
strong effect on convergence. This is verified in controlled simulations.
The theoretical details are discussed in Section \ref{sec:inference}.

\section{Preliminaries}
\subsection{Setup and Notation}
We let $H$ and $I$ denote respectively differential entropy, and mutual information \cite{cover1994elements}.
Given a density $p$ we will at times use the (abuse of) notation $H(p)$ when a r.v. is unspecified.

The distribution of a r.v. $Z$ is denoted $P_Z$, and its density when it exists is denoted
$p_Z$.

Throughout the analysis we will be concerned with residuals from regression fits. We use the
following notation.
\begin{definition}
 For a function $f: \real\mapsto \real$, we consider either of the {\bf residuals}:
$\eta_{Y, f} \triangleq Y - f(X)$ and $\eta_{X, f} \triangleq X - f(Y)$.
\end{definition}

The Causal Additive Noise Model is captured as follows:

\begin{definition}[CAM]
\label{def:inass}
Given r.v.'s $X, Y$, a function $f:\real\mapsto\real$ and a r.v. $\eta$, we write $X\xrightarrow{f, \eta} Y$ if the following holds:
\begin{itemize}
 \item[({\rm{i}})]$P_{X, Y}$ is generated as $X\sim P_X$, and
$Y=f(X) + \eta $, where the noise r.v. $\eta$ has $0$ mean and $\eta\ci X$;
\item[({\rm{ii}})] for any $g:\real\mapsto \real$, $\eta_{X,g}\triangleq X-g(Y)$ depends on $X$.

\end{itemize}
%
%

We write $X\to Y$ when $f$ and $\eta$ are clear from context.
\end{definition}


\section{Causal Inference Procedures}
\label{sec:algorithm0}
\subsection{Main Intuition}
\begin{lemma}
\label{lem:complexity}
Consider any absolutely continuous joint-distribution $P_{X, Y}$ on $X, Y\in \real$. For any two functions $f, g: \real \mapsto \real$
we have
\begin{align*}
 H(X) + H(\eta_{Y, f}) &= H(Y) + H(\eta_{X, g}) \\
 &\quad -\braces{I(\eta_{X, g}, Y) - I(\eta_{Y, f}, X)}.
\end{align*}
\end{lemma}
\begin{proof}
 By the chain rule of differential entropy we have
 \begin{align*}
  H(X, Y) &= H(X) + H(Y| X) = H(X) + H(\eta_{Y, f}| X) \\
  &= H(X) + H(\eta_{Y, f}) - I(\eta_{Y, f}, X), \text{ similarly }
 \end{align*}
 \begin{align*}
  H(X, Y) = H(Y) + H(\eta_{X, g}) - I(\eta_{X, g}, Y).
 \end{align*}
Equate the two r.h.s above and rearrange.
\end{proof}

Note that whenever $\eta_{Y, f}, \ci X$, we have $I(\eta_{Y, f}, X)=0$.
Therefore, by the above lemma, if $\eta_{Y, f}, \ci X$
then $C_{XY}\triangleq H(X) + H(\eta_{Y, f})$ is
smaller than $C_{YX}\triangleq H(Y) + H(\eta_{X, g})$. This yields a measure of independence
which is relatively cheap to estimate. In particular the test depends only on the marginal
distributions of the r.v.'s $X, Y$ and functional residuals, and does not involve estimating
joint distributions or conditionals, as is implicit in most independence tests.
We analyze a family of procedures based on this idea. This family is given in the next subsection.

\subsection{Meta-Algorithm}
\label{sec:algorithm}
Let $\braces{(X_i, Y_i)}_1^n = \braces{(x_1, y_1), \dots, (x_n, y_n)}$ be a finite sample drawn from $P_{X, Y}$.
Let $H_n(X)$ and $H_n(Y)$ be respective estimators of $H(X)$ and $H(Y)$ based on the sample
$\braces{(X_i, Y_i)}_1^n$.

We consider the following family of inference procedures:

\begin{quote}
\footnotesize
Given an i.i.d sample $\braces{(X_i, Y_i)}_1^n$ from $P_{X, Y}$, let $f_n$ be returned by an algorithm which fits $Y$ as $f_n(X)$ and
$g_n$ be returned by an algorithm which fits $X$ as $g_n(Y)$. Let $H_n$ denote an
entropy estimator. Given a threshold parameter $\tau_n\xrightarrow{n\to \infty}0$:

Decide $X\to Y$ if
$$H_n(X) + H_n(\eta_{Y, f_n}) +\tau_n \leq H_n(Y) + H_n(\eta_{X, g_n}).$$
Decide $Y\to X$ if
$$H_n(Y) + H_n(\eta_{X, g_n}) +\tau_n \leq H_n(X) + H_n(\eta_{Y, f_n}).$$
Abstain otherwise.
\end{quote}

The analysis in this paper is carried with respect to the $L_{2, P_X}$  and $L_{2, P_Y}$
 functional norms defined as follows.

\begin{definition}
For $f:\real\mapsto \real$, and a measure $\mu$ on $\real$,
the $L_{2, \mu}$ norm is given as
$\norm{f}_{2, \mu}=\paren{\int_t f(t)^2 \, d\mu(t)}^{1/2}.$
\end{definition}

%

We assume the internal procedures $f_n, g_n, H_n$ have the following
consistency properties.

\begin{assumption}
 \label{ass:consistency} The internal procedures are consistent:
 \begin{itemize}
  \item Suppose $\expec  Y^2 <\infty$. Let $f(x) \triangleq \expec [Y|x]$.
  Then $\norm{f_n - f}_{2, P_X}\xrightarrow{P} 0$.
  \item Suppose $\expec X^2 < \infty$. Let $g(y) \triangleq \expec[X|y]$.
  Then $\norm{g_n - g}_{2, P_Y}\xrightarrow{P}0$.
  \item Suppose $Z$ has bounded variance, and has continuous density $p_Z$ such that 
  $\exists \, T, C>0, \alpha>1, \quad \forall \abs{t}>T, \, p_Z(t)\leq C\abs{t}^{-\alpha}$.   
   Then
  $\abs{H_n(Z)-H(Z)}\xrightarrow{P} 0$.
 \end{itemize}
\end{assumption}

Many common nonparametric regression procedures (e.g. kernel, $k$-NN, Kernel-SVM, spline regressors) are
consistent in the above sense \cite{GKKW:81}. Also the consistency of a variety
of entropy estimators (e.g. plug-in entropy estimators) is well established \cite{beirlant1997nonparametric}.

\section{Technical overview of results}
\label{sec:inference} 

\begin{figure*}[t]
\centering
\subfigure[]
  {\includegraphics[height=1.4in]{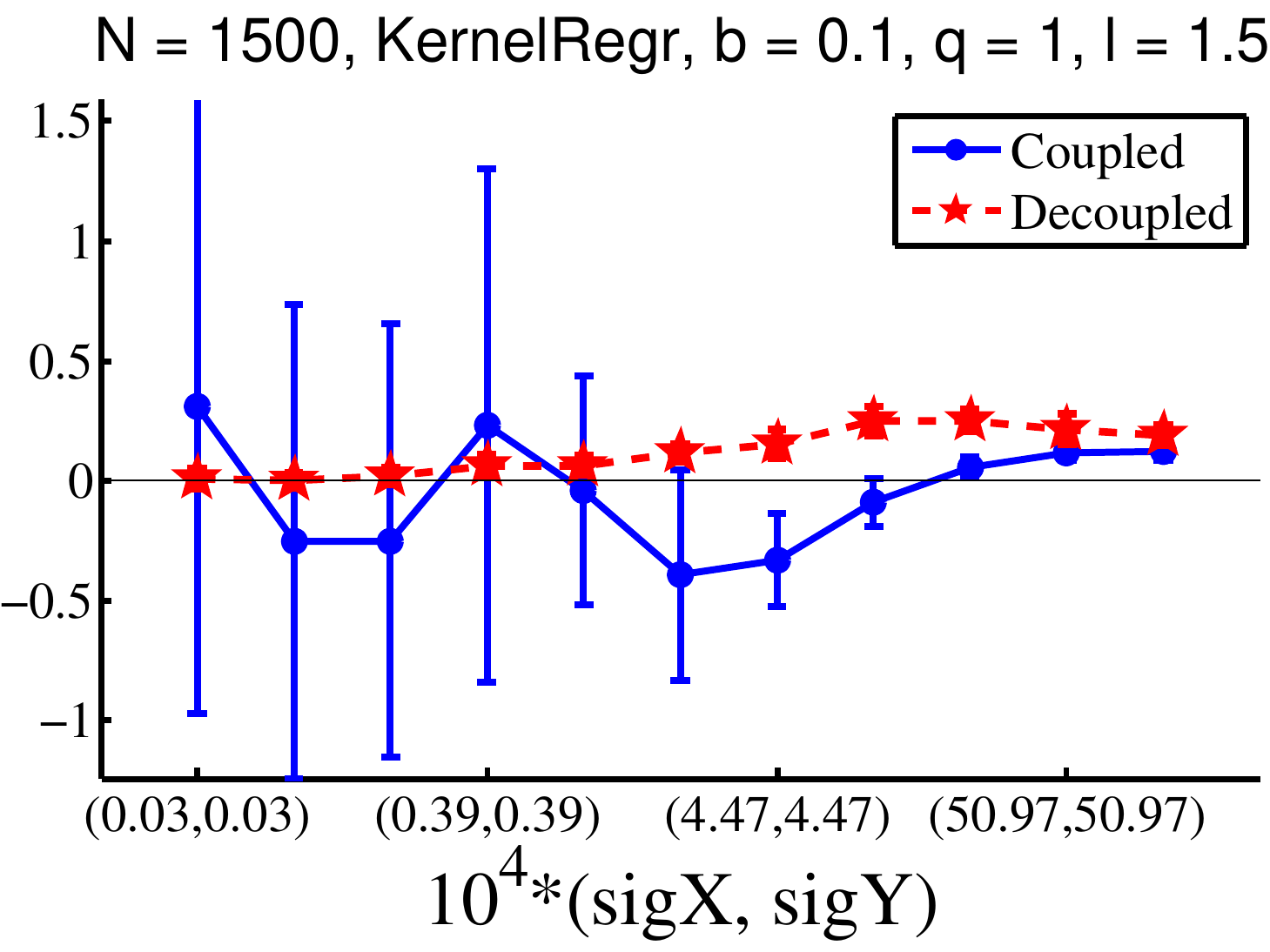}
  \label{fig:Sig_KerRegr_b=0.1}}
\subfigure[]
  {\includegraphics[height=1.4in]{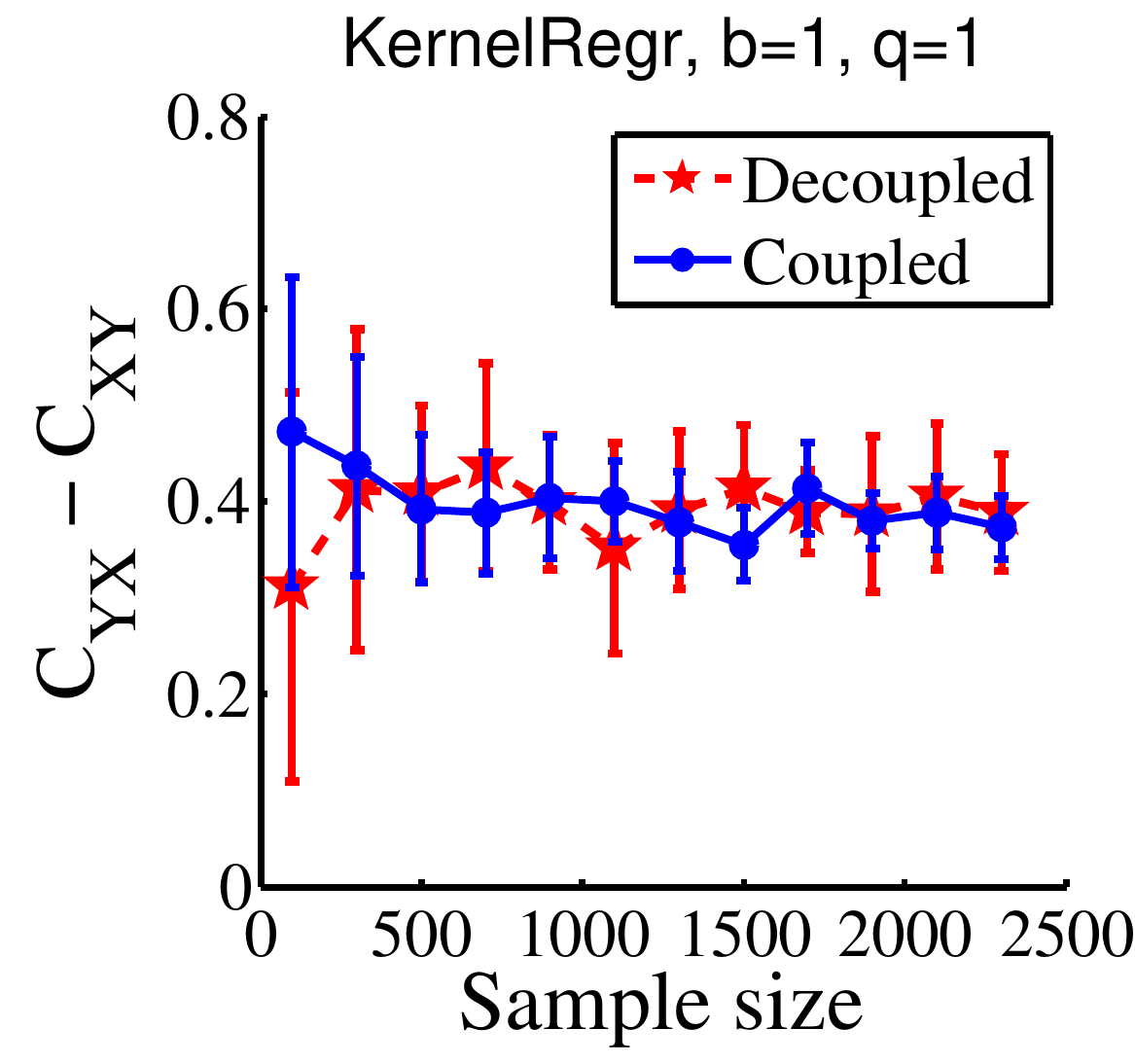}
  \label{fig:N_KerRegr}}
\subfigure[]
  {\includegraphics[height=1.4in]{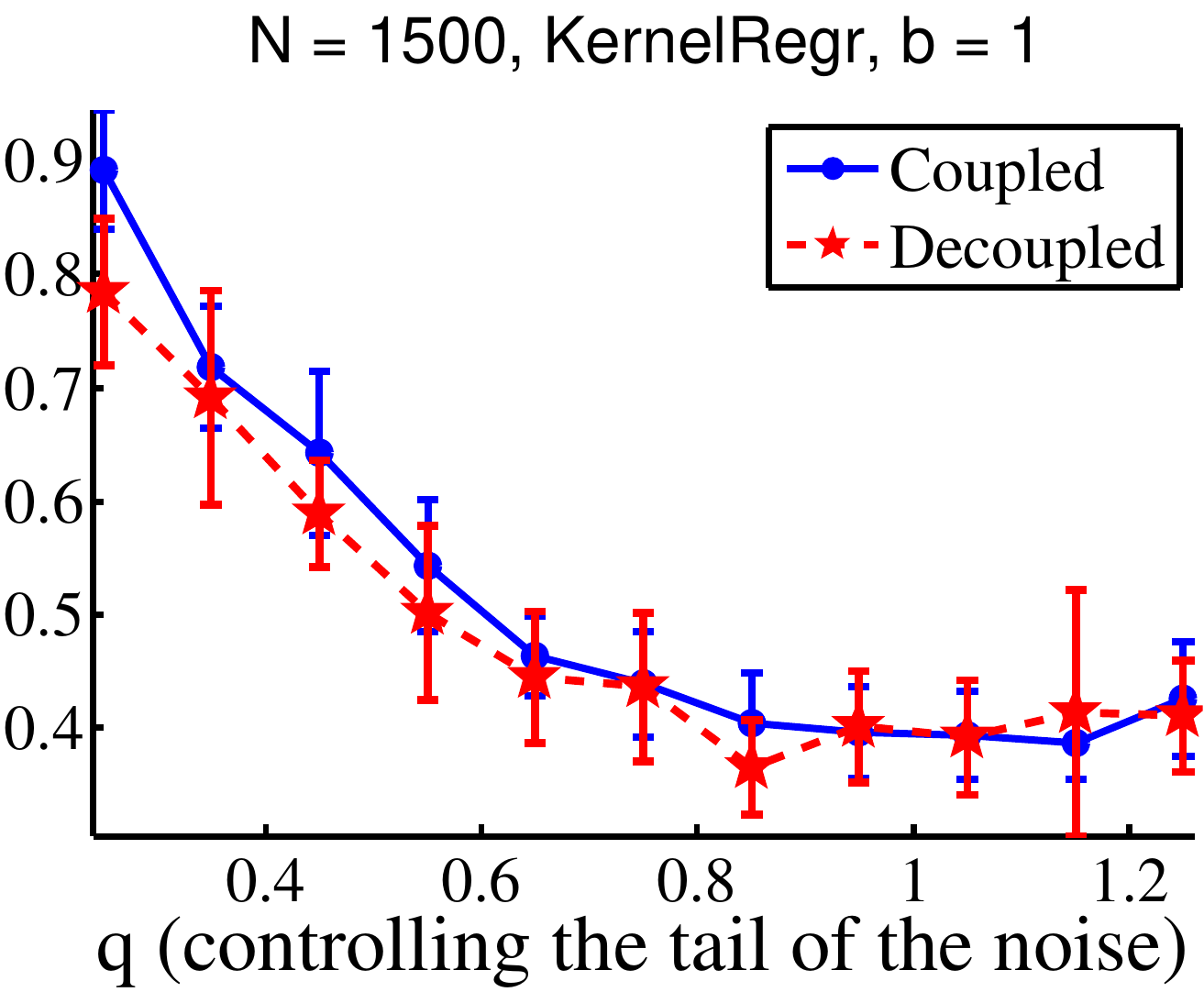}
  \label{fig:Q_KerRegr}}
\caption[font=small,labelsep=none] { \label{results} Plots of the difference between the complexity measures ($C_{YX}-C_{XY}$)
for coupled and decoupled-estimation in various scenarios. Simulated data is generated as $Y = bX^3 + X + \eta$. $X$
is sampled from a uniform distribution on the open interval $(-2.5,2.5)$, while $\eta$ is sampled
as $\abs{\mathcal{N}}^q\cdot \text{sign}(\mathcal{N})$ where $\mathcal{N}$ is a standard normal.
$b$ controls the strength of the nonlinearity of the function and $q$ controls the non-Gaussianity of the noise: $q = 1$ gives a Gaussian, while
$q > 1$ and $q < 1$ produces super-Gaussian and sub-Gaussian distributions, respectively.
For entropy estimation we employ a resubstitution estimate using a kernel density estimator tuned against log-likelihood
\cite{beirlant1997nonparametric} and
for regression estimator we use kernel regression (KR).
For every combination of the parameters, each experiment was repeated $10$ times, and average results for ($C_{YX}-C_{XY}$)
are reported along with standard deviation across repetitions.
Plot (a): increasing kernel bandwidth of regressor geometrically (by factors of $l=1.5$), i.e. decreasing richness of the algorithm.
When the capacity of the regression algorithm is too large,
the variance of the causal inference is large for coupled-estimation (due to overfitting) but remains low for decoupled-estimation.
Plot (b): increasing sample size (bandwidth of KR tuned by cross-validation). For tuned bandwidth, the variance of
the causal inference is only due to the sample size, so the coupled-estimation (which estimates
everything on a larger sample) becomes the better procedure.
Plot (c): increasing $q$, i.e. the tail of the noise is made sharper (KR tuned by cross-validation). For
faster decreasing tail of the noise, the causal inference becomes better.
The experiments of Figures (b) and (c) were repeated using kernel ridge regression (KRR) tuned by cross-validation
(see supplementary appendix).
For properly tuned parameters, the selection of regression method does not seem to matter for the causal inference results.}
\end{figure*}

We consider the following two versions of the above meta-procedure.
The analysis (Section \ref{sec:analysis}) is divided accordingly.

\begin{definition}[Decoupled-estimation]
\label{def:decoupled}
 $f_n$ and $g_n$ are learned on half of the sample $\braces{(X_i, Y_i)}_1^n$,
and the $H_n\paren{\eta_{Y, f_n}}$ and $H_n\paren{\eta_{X, g_n}}$ are learned on the
other half of the sample (w.l.o.g. assume $n$ is even). $H_n(X)$ and $H_n(Y)$ could be
learned on either half or on the entire sample.
\end{definition}

\begin{definition}[Coupled-estimation]
 All $f_n$, $g_n$ and entropies $H_n$ are learned on the entire sample $\braces{(X_i, Y_i)}_1^n$.
\end{definition}

Our most general consistency result (Theorem \ref{theo:relating_Rp_H}, Section \ref{sec:decoupled}) concerns decoupled-estimation.
By decoupling regression and entropy estimations, we
reduce the potential of overfitting, during entropy estimation, the generalization error of regression.
This generalization error could be large if the regression algorithms are too rich (e.g. ERM over large functional classes).
Our simulations show that, when the regression algorithm is too rich, the variance of the causal inference is large for coupled-estimation
but remains low for decoupled-estimation (Fig. \ref{fig:Sig_KerRegr_b=0.1}). By decreasing the richness of the class (simulated by increasing the
kernel bandwidth for a kernel regressor) the source of variance shifts to the sample size, and coupled-estimation (which estimates
everything on a larger sample) becomes the better procedure and tends to converge faster
(Fig. \ref{fig:N_KerRegr}).

For the consistency result of Theorem \ref{theo:relating_Rp_H} we make no assumption on the richness of the regression algorithms,
but simply assume that they converge in $L_2$ (Assumption \ref{ass:consistency}). The main technicality is to then show that entropies
of residuals are locally continuous relative to the $L_2$ metric in both causal and anticausal directions.

For coupled-estimation, the main difficulty is the following. Even though the entropy estimators are consistent for a \emph{fixed} distribution,
the distribution of the residuals change with $f_n$ and $g_n$, thus with every random sample
(this problem is alleviated by decoupling the estimation). However, if the richness of the regression algorithms is controlled, in
other words if the set of potential $f_n$ and $g_n$ is not too rich, then the entropy estimate for residuals might converge. We show
in Theorem \ref{theo:coupled} (Section \ref{sec:coupled}) that if we employ kernel regressors with properly chosen bandwidths, and kernel-based entropy
estimators with sufficiently smooth kernels, then the resulting method is consistent for causal inference.

Both consistency results of Theorem \ref{theo:relating_Rp_H} and Theorem \ref{theo:coupled} rely on tail assumptions on
the additive noise $\eta$ (where $X\xrightarrow{f, \eta}Y$). We assume an exponentially decreasing tail for the more
difficult case of coupled-estimation, but need only a mild assumption of polynomially decreasing tail in the case
of decoupled-estimation. Note that it is common to assume that $\eta$ has Gaussian tail, and our assumptions are milder
in that respect.

Interestingly, our analysis for Theorem \ref{theo:relating_Rp_H} suggests that convergence of causal inference is likely faster if the noise
$\eta$ has faster decreasing tail
(see Lemma \ref{lem:propOfResiduals1}). This is verified in our simulations where we vary
the tail of $\eta$ (Fig. \ref{fig:Q_KerRegr}).

\section{Analysis}
\label{sec:analysis}
\subsection{Consistency for Decoupled-estimation}
\label{sec:decoupled}
In this section we establish a general consistency result for the meta-procedure
above. The main technicality consists of relating differential entropy of residuals to the $L_2$-norms of residuals (i.e. to the
error made in function estimation). We henceforth let $\Sigma$ denote the Lebesgue measure.

The analysis in this section uses the following polynomial tail assumption on $\eta$.
We note that Assumption \ref{ass:continuityofP} satisfies the idenfiability conditions of
\cite{zhang2009identifiability}.

\begin{assumption}[Tail]
\label{ass:continuityofP}
$P_{X, Y}$ is generated as follows: $X\xrightarrow{f, \eta} Y$ for some bounded function $f$,
with bounded derivative on $\real$.
$P_X$ has bounded support, and both $P_X$ and $P_\eta$ have densities $p_X$, $p_\eta$
with bounded derivatives on $\real$.  Furthermore, we assume $\eta$ has bounded variance, and $p_\eta$ satisfies,
for some $T>0$, $C>0$, and $\alpha>1$:
\begin{align}
 \forall \abs{t}>T, \quad p_\eta(t) \leq C\abs{t}^{-\alpha}.\label{eq:tailassumption}
\end{align}
\end{assumption}

Note that, since the unknown target functions are assumed bounded,
any consistent regressor can be appropriately truncated while maintaining consistency. We therefore
have the following technical assumption on the regressors.

\begin{assumption}
 \label{ass:boundedness}
The regression procedures return bounded functions: 
 $\lim_{n\to \infty} \max\braces{\norm{f_n(t)}_\infty, \norm{g_n(t)}_\infty} <\infty$.
\end{assumption}

\begin{theorem}[General consistency for decoupled-estimation]
Suppose $X\xrightarrow{f, \eta} Y$ for some $f, \eta$,
and $P_{X, Y}$ satisfies the tail Assumption \ref{ass:continuityofP}.
Suppose $f_n$, $g_n$, and $H_n$ are consistent procedures satisfying Assumption \ref{ass:consistency} and \ref{ass:boundedness}.
Let the meta-algorithm be decoupled as in Definition \ref{def:decoupled}.

Then the probability of correctly deciding $X\to Y$ goes to $1$ as $n\to \infty$.
\label{theo:relating_Rp_H}
\end{theorem}

To prove the theorem, we have to understand how the estimated entropies converge as a function
of the $L_2$ error in regression estimation. We will proceed by bounding the distance between
the densities $p_{\eta_{Y, f}}$ and $p_{\eta_{Y, f'}}$ of the residuals of functions $f$ and $f'$
in terms of the $L_2$ distance between $f$ and $f'$ (Lemma \ref{lem:propOfResiduals1});
this will then be used to bound the difference in the entropy of such residuals.

Given Assumption \ref{ass:continuityofP}, the following lemma establishes some useful properties
of the distribution $P_{X, Y}$ and of the distribution of certain residuals.
It is easy to verify that under our assumptions, all distributions under consideration in the lemma are absolutely continuous.

\begin{lemma}[Properties of induced densities]
\label{lem:propInduced}
 Suppose $P_{X, Y}$ satisfies Assumption \ref{ass:continuityofP} for some $f, \eta$, and $\alpha>1$.
 We then have the following: \\
 ({\rm{i}}) $p_{X, Y}$ has a bounded gradient on $\real^2$,\\
 ({\rm{ii}}) consider functions $f', g: \real\mapsto \real$ and suppose $\sup\abs{f'}$ and $\sup\abs{g}$ are at most $T_0$ for some $T_0$;
 then there exists $T'>0$ depending on $T_0$, and $C'>0$ such that $\forall \abs{t}>T'$
 \begin{align*}
\braces{p_{X, Y}(\cdot, t), p_{X, Y}(t, \cdot), p_{\eta_{Y, f'}}(t), p_{\eta_{X, g}}(t)} \leq C' \abs{t}^{-\alpha}. 
 \end{align*}
In particular, the above holds for $g(y)\triangleq \expec[X|Y=y]$.
\end{lemma}

The next lemma relates the density of residuals to the $L_2$ distance between functions. Notice, as discussed in
Section \ref{sec:inference}, that the Lemma suggests that the densities of residuals converge faster the sharper the
tails of the noise $\eta$: the larger $\alpha$, the sharper the bounds are in terms of the $L_2$ distance between
functions.

\begin{lemma}[Density of residuals w.r.t. $L_2$ distance]
\label{lem:propOfResiduals1}
Suppose the joint distribution $P_{X, Y}$ satisfies Assumption \ref{ass:continuityofP} for
some $f, \eta$ and $\alpha>1$. Let $g(y)\triangleq \Expectation[X| Y=y]$.
Consider functions $f', g':\real\mapsto\real$. There exist a constant $C''$ such that
for $\norm{f-f'}_{2, P_X}$ and (respectively) $\norm{g-g'}_{2, P_Y}$ sufficiently small, we have
\begin{align*}
 \sup_{t\in \real}\abs{p_{\eta_{Y, f'}}(t) -p_{\eta_{Y, f}}(t)}&\leq C''\paren{\norm{f'-f}_{2, P_X}}^{(\alpha -1)/2\alpha},\\
 \text{ and} \\
 \sup_{t\in \real}\abs{p_{\eta_{X, g'}}(t) -p_{\eta_{X, g}}(t)}&\leq C''\paren{\norm{g'-g}_{2, P_Y}}^{(\alpha -1)/2\alpha}.
\end{align*}
\end{lemma}
\begin{proof}
We start by bounding the difference between $p_{\eta_{Y, f'}}(t)$ and $p_{\eta_{Y, f'}}(t)$.
We note that the same ideas can be used to bound the difference
between $p_{\eta_{X, g'}}(t)$ and $p_{\eta_{X, g}}(t)$, since $X$ and $Y$ are interchangeable
in the analysis from this point on. This is because what follows does
not depend on how $P_{X, Y}$ is generated, just on the properties of the induced distributions
as stated in Lemma \ref{lem:propInduced}.

We will partition the space $\real$ as follows. First, let
$\real_>$ denote the set 
$$\braces{x: \abs{f(x) - f'(x)}>\sqrt{\norm{f-f'}_{2, P_X}}}.$$
We define the following interval $\D\subset \real$: let $T'$ be defined as in Lemma \ref{lem:propInduced},
and $\tau>T'$; we have
$\D\triangleq [-\tau, \tau]$.
%

For any $t\in \real$ we have by writing residual densities in terms of the joint
$p_{X, Y}$
(as in the proof of Lemma \ref{lem:propInduced} in supplementary appendix) that
$\abs{p_{\eta_{Y, f'}}(t) -p_{\eta	}(t)}$
\begin{align}
  &= \abs{\int_\real \paren{p_{X, Y}(x, t+f'(x)) - p_{X,Y}(x,t+f(x))} \, dx}\nonumber\\
 &\leq \abs{\int_{\real\setminus D}\paren{p_{X, Y}(x, t+f'(x)) - p_{X,Y}(x,t+f(x))} \, dx} \label{eq:decomp_0}\\
 &\, + \int_{\D\setminus \real_>} \abs{p_{X, Y}(x, t+f'(x)) - p_{X,Y}(x,t+f(x))} \, dx \label{eq:decomp_p}\\
 &\, + \abs{\int_{\real_>} \paren{p_{X, Y}(x, t+f'(x)) - p_{X,Y}(x,t+f(x))} \, dx}.\label{eq:decomp_p2}
\end{align}

To bound the first term (\ref{eq:decomp_0}), let $y_x$ denote either of $t+ f'(x)$ or $t+f(x)$, we have by
Lemma \ref{lem:propInduced} that
\begin{align*}
 \int_{\tau}^\infty p_{X, Y}(x, y_x)\, dx \leq \int_{\tau}^\infty C'x^{-\alpha}\, dx\leq \frac{C'}{\alpha - 1} \tau^{-(\alpha -1)},
\end{align*}
so that the first term (\ref{eq:decomp_0}) is at most $2\frac{C'}{\alpha - 1} \tau^{-(\alpha -1)}$.

%
To bound the second term (\ref{eq:decomp_p}) we recall that $p_{X, Y}$ has a bounded gradient on $\real^2$
(Lemma \ref{lem:propInduced}). Therefore there exists $C_0$ such that
for every $x, y, \epsilon \in \real$,
$p_{X, Y}(x, y+ \epsilon)$ differs from $p_{X, Y}(x, y)$ by at most $C_0\cdot\abs{\epsilon}$.
It follows that the second term (\ref{eq:decomp_p}) is at most
\begin{align*}
\int_{\D\setminus \real_>} C_0\abs{f'(x) - f(x)} \, dx
\leq 2\tau \cdot C_0{\sqrt{\norm{f-f'}_{2, P_X}}}.
\end{align*}

The third term (\ref{eq:decomp_p2}) is equal to
\begin{align*}
&| \pr{X\in \real_>, Y=t+f'(X)} - \\
& \quad \pr{X\in \real_>, Y=t+f(X)} |
\leq P_X(\real_>).
\end{align*}

We next bound
$P_X\paren{\real_>}$ while noting that ${\norm{f-f'}_{2, P_X}}$ could be $0$. Let
$\epsilon>{\norm{f-f'}_{2, P_X}}$. By Markov's inequality,
\begin{align*}
 P_X\braces{\abs{f(X) - f'(X)}>\sqrt{\epsilon}}&\leq \frac{\norm{f-f'}_{1, P_X}}{\sqrt{\epsilon}}\\
 &\leq \frac{\norm{f-f'}_{2, P_X}}{\sqrt{\epsilon}}.
\end{align*}
Thus, consider a sequence of $\epsilon\to\norm{f-f'}_{2, P_X}$, by Fatou's lemma we
have $P_X\paren{\real_>}\leq \sqrt{\norm{f-f'}_{2, P_X}}$.

Combining the above analysis we have that
\begin{align*}
 \abs{p_{\eta_{Y, f'}}(t) -p_{\eta}(t)}\leq& 2\frac{C'}{\alpha - 1} \tau^{-(\alpha -1)} \\
 &+ (1+2\tau \cdot C_0){\sqrt{\norm{f-f'}_{2, P_X}}}.
\end{align*}
Now, for $\norm{f-f'}_{2, P_X}$ sufficiently small, we can pick $\tau = O\paren{\norm{f-f'}_{2, P_X}}^{-1/2\alpha}$ to get the result.

As previously noted we can use the same ideas as above to
similarly bound $\abs{p_{\eta_{X, g'}}(t) -p_{\eta_{X, g}}(t)}$ for all $t\in \real$.
It suffices to interchange $X$ and $Y$ in the above analysis.
\end{proof}

\begin{lemma}
\label{lem:entropy}
Let $p_1, p_2$ be two densities such that there exist $T, C>1$ and $\alpha>1$, for all
$\abs{t}>T$, $\max_{i\in [2]}p_i(t)<C\abs{t}^{-\alpha}$. Suppose
$\sup_{t\in \real}\abs{p_1(t) - p_2(t)}<\epsilon$ for some $\epsilon<\min\braces{1/T^2, 1/(3 e)}$ satisfying the
further condition: $\forall t>1/\sqrt{\epsilon}, \, t^{(\alpha-1)/2}>\ln t$.
We then have for $\epsilon$ sufficiently small
\begin{align*}
 \abs{H(p_1)-H(p_2)}\leq 18\sqrt{\epsilon}\ln(1/3\epsilon) + \frac{4C\alpha}{\alpha - 1}\epsilon^{(\alpha - 1)/4}.
\end{align*}
\end{lemma}
\begin{proof}
For simplicity of notation in what follows, let $\tau\triangleq1/\sqrt{\epsilon}$.
Let $\D\triangleq [-\tau, \tau]$ and let $\D_{2>}\triangleq \braces{t\in \D, p_2(t)>2\epsilon}$. Define $\gamma(u) = -u\ln u$ for $u>0$, and $\gamma(0) = 0$.
We will use the fact that for the function $\gamma(\cdot)$ is increasing on $[0, 1/e]$. We have 
\begin{align}
H(p_1)
=& \int_{\real\setminus\D}\gamma(p_1(t)) \, dt  +
\int_{\D_{2>}}\gamma(p_1(t))\, dt \nonumber\\
&+ \int_{\D\setminus \D_{2>}}\gamma(p_1(t))\, dt \nonumber\\
\leq& \int_{\real\setminus\D}\gamma(p_1(t)) \, dt + \int_{\D_{2>}}p_1(t)\ln \frac{1}{p_1(t)}\, dt  \nonumber\\
&+ \Sigma\paren{\D\setminus \D_{2>}} \cdot\gamma(3\epsilon), \label{eq:Hp1_decomp}
 \end{align}
since for $t\in \D\setminus \D_{2>}$ we have $$p_1(t) \leq p_2(t) + \epsilon\leq 3\epsilon\leq1/e.$$

To bound the first term of (\ref{eq:Hp1_decomp}), notice that
\begin{align*}
 \int_{\tau}^\infty \gamma(p_1(t))\, dt &\leq \int_{\tau}^\infty -Ct^{-\alpha}\ln\paren{Ct^{-\alpha}}\, dt\\
 &\leq  \int_{\tau}^\infty C\alpha t^{-\alpha}\ln t \, dt\\
 &\leq \int_{\tau}^\infty C\alpha t^{-(\alpha + 1)/2}\, dt \\
 &\leq \frac{2C\alpha}{\alpha - 1} \tau^{-(\alpha -1)/2},
\end{align*}
hence we have  $$\int_{\real\setminus\D}\gamma(p_1(t)) \, dt\leq C'\tau^{-\alpha'} \text{ for } C', \alpha'>0.$$

Next we bound the
second term of (\ref{eq:Hp1_decomp}) as follows:
\begin{align*}
 \int_{\D_{2>}}& p_1(t)\ln \frac{1}{p_1(t)}\, dt \\
 &\leq \int_{\D_{2>}}(p_2(t) + \epsilon)\ln \frac{1}{p_2(t) - \epsilon}\, dt\\
 &=\int_{\D_{2>}}p_2(t)\ln\frac{1}{p_2(t)(1-\epsilon/p_2(t))}\, dt \\
 &+ \int_{\D_{2>}}\epsilon\ln \frac{1}{p_2(t) - \epsilon}\, dt\\
&\leq H(p_2) +  \int_{\D_{2>}}p_2(t)\ln\frac{1}{1-\epsilon/p_2(t)}\, dt \\
&+ \int_{\D_{2>}}\epsilon\ln\frac{1}{\epsilon}\, dt\\
&\leq H(p_2) + \int_{\D_{2>}}p_2(t)\ln(1+2\epsilon/p_2(t))\, dt \\
&+ \Sigma\paren{\D_{2>}}\cdot\gamma(\epsilon) \\
&\leq H(p_2) + 2\Sigma\paren{\D_{2>}}\cdot \epsilon + \Sigma\paren{\D_{2>}}\cdot\gamma(\epsilon).
\end{align*}
Combining all the above, we have
\begin{align*}
 H(p_1) \leq& H(p_2) + 3\Sigma\paren{\D}\cdot \gamma(3\epsilon) + C'\tau^{-\alpha'}\\
 =& H(p_2) + 18\sqrt{\epsilon}\ln(1/3\epsilon) + C'\epsilon^{\alpha'/2}.
\end{align*}

Notice that $p_1$ and $p_2$ are interchangeable in the above argument. The result therefore follows.
\end{proof}


We are now ready to prove the main theorem.

\subsection*{Proof of Theorem \ref{theo:relating_Rp_H}}
 Let $f(x)\triangleq\expec[Y|x]$ and $g(y)\triangleq\expec[X|y]$.
  By Lemma \ref{lem:complexity},
  \begin{align}
   H(X) + H\paren{\eta_{Y, f}} > H(Y) + H\paren{\eta_{X, g}} + 8\epsilon,
  \end{align}
  for some $\epsilon>0$.

Thus we detect the right direction $X\to Y$ if all quantities
(a) $\abs{H_n\paren{\eta_{Y, f_n}} - H\paren{\eta_{Y, f}}}$,
(b) $\abs{H_n\paren{\eta_{X, g_n}} - H\paren{\eta_{Y, g}}}$,
(c) $\abs{H_n(X) - H(X)}$, and (d) $\abs{H_n(Y) - H(Y)}$, are at most $\epsilon$.

By assumption, (c) and (d) both tend to $0$ in probability.
The quantities (a) and (b) are handled as follows. We only show the argument for (a),
as the argument for (b) is the same. We have:
\begin{align*}
\abs{H_n\paren{\eta_{Y, f_n}} - H\paren{\eta_{Y, f}}} & \leq
\abs{H_n\paren{\eta_{Y, f_n}} - H\paren{\eta_{Y, f_n}}} \\
& \quad + \abs{H\paren{\eta_{Y, f_n}} - H\paren{\eta_{Y, f}}}.
\end{align*}
Now $H_n\paren{\eta_{Y, f_n}}$ is consistent for $f_n$ fixed (it easy to check that
$P_{\eta_{Y, f_n}}$ satisfies the necessary conditions provided $f_n$ is bounded)
and $f_n$ is learned on an independent sample from $H_n$, we have
$\abs{ H_n\paren{\eta_{Y, f_n}} - H\paren{\eta_{Y, f_n}}} \xrightarrow{P} 0$.

By Lemma \ref{lem:propOfResiduals1}, convergence of $f_n$ i.e. $\norm{f_n -f}_{2, P_X}\xrightarrow{P}0$
implies $\sup_{t}\abs{p_{\eta_{Y, f_n}}(t) - p_{\eta_{Y, f}}(t)}\xrightarrow{P} 0$; this in turn implies
by Lemma \ref{lem:entropy} that $\abs{H\paren{\eta_{Y, f_n}} - H\paren{\eta_{Y, f}}}\xrightarrow{P} 0$.

Thus all quantities (a)-(d) are at most $\epsilon$ with probability going to $1$.
\qed


\subsection{Coupled Regression and Residual-entropy Estimation}
\label{sec:coupled}
Here we consider a coupled version of the meta-algorithm where
$f_n$ and $g_n$ are kernel regressors. This is described in the next subsection.
\subsubsection{Kernel instantiation of the meta-algorithm}
\label{sec:kerneldefs}
{\bf Regression:} Although any kernel that is $0$ outside a bounded region will work for the regression,
we focus here (for simplicity) on the particular case where $f_n$ and $g_n$ are
box-kernel regressors defined as follows (interchange $X$ and $Y$ to obtain $g_n(y)$):
\begin{align}
&f_n(x)  = \frac{1}{n_{x, h}}\sum_{i=1}^n Y_i \ind{\abs{X_i - x} < h}, \label{eq:boxkernel}\\
&\text{where } n_{x, h} = \abs{i: \abs{X_i - x} < h}, \text{ for a bandwidth } h. \nonumber
\end{align}

{\bf Entropy estimation:} Given a sequence $\epsilon = \braces{\epsilon_i}_{i =1}^n$, and
a bandwidth $\sigma$, define $p_{n, \epsilon}$
as follows:
\begin{align*}
 p_{n, \epsilon}(t) =& \frac{1}{nh}\sum_{i=1}^n K\paren{\frac{{\epsilon_i - t}}{\sigma}},\\
 &\text{ where } \int_{\real}K(u)\, du = 1,\, \abs{\frac{d}{d u} K(u)}< \infty,\\
&\text{ and } K(u)=0 \text{ for } \abs{u}\geq 1.
\end{align*}

Let $\epsilon_{Y, i} = Y_i - f_n(X_i)$ and $\epsilon_{X, i} = X_i - g_n(Y_i)$.
The residual entropy estimators are defined as:
\begin{align}
 H_n\paren{\eta_{Y, f_n}} \triangleq H\paren{p_{n, \epsilon_{Y}}} \text{ and }
 H_n\paren{\eta_{X, g_n}} \triangleq H\paren{p_{n, \epsilon_{X}}}. \label{eq:entropy}
\end{align}

\subsubsection{Consistency result for coupled-estimation}
We abuse notation and use $h$ and $\sigma$ to denote the bandwidth parameters used to estimate
either $f_n$ and $H_n\paren{\eta_{Y, f_n}}$, or $g_n$ and $H_n\paren{\eta_{X, g_n}}$. We make the distinction
clear whenever needed.

The consistency result depends on the following quantities bounded in Lemma \ref{lem:Rn}.

\begin{definition}[Expected average excess risk]
 Define $R_{n}(f_n)\triangleq {\expec\frac{1}{n}\sum_{i=1}^n \abs{f_n(X_i) - f(X_i)}}$ and similarly
$R_n(g_n)\triangleq {\expec\frac{1}{n}\sum_{i=1}^n \abs{g_n(Y_i) - g(Y_i)}}$.
\end{definition}

We assume in this section that the noise $\eta$ has exponentially decreasing tail:
\begin{definition}
A r.v. $Z$ has {\bf exponentially decreasing tail} if there exists $C, C'>0$ such that
for all $t>0, \quad \pr{\abs{Z-\expec Z}>t}\leq Ce^{-C't}$.
\end{definition}

The following consistency theorem hinges on properly choosing the bandwidths parameters $h$ and $\sigma$.
Essentially we want to choose $h$ such that regression estimation is consistent, and we want to choose
$\sigma$ so as not to overfit regression error. If the bandwidth $\sigma$ is too small relative to regression
error (captured by $R_n$), then the entropy estimator (for the residual entropy) is only fitting this error.
The conditions on $\sigma$ in the Theorem are mainly to ensure that
$\sigma$ is not too small relative to regression error $R_n$.

\begin{theorem}[Coupled estimation]
\label{theo:coupled}
Suppose $X\xrightarrow{f, \eta} Y$ for some $f, \eta$,
and suppose $P_{X, Y}$ satisfies Assumption \ref{ass:continuityofP}, and $\eta$ has exponentially decreasing tail.
Let $f_n$, $g_n$, and $H_n$ be defined as in Section \ref{sec:kerneldefs},
and let both $H_n(X)$ and $H_n(Y)$ be consistent as in Assumption \ref{ass:consistency}.

Suppose that : \\
({\rm{i}}) For learning $f_n$ and $H_n\paren{\eta_{Y, f_n}}$, we use $h=c_1n^{-\alpha}$ for some
$c_1>0$ and $0<\alpha<1$, and $\sigma=c_2 n^{-\beta}$ for some $c_2>0$ and
$0<\beta< \min\braces{(1-\alpha)/4, \alpha/2}$.\\
({\rm{ii}}) For learning $g_n$ and $H_n\paren{\eta_{X, g_n}}$,
$h$ satisfies $h\to0$ and $nh\to\infty$, and $\sigma$ satisfies $\sigma\to0$,
$n\sigma\to\infty$, and $\sigma=\Omega(R_n(g_n)^{-\gamma})$
for some $0<\gamma<1/2$.

Then the probability of correctly detecting $X\to Y$ goes to $1$ as $n\to \infty$.
\end{theorem}

The theorem relies on Lemma \ref{lem:Rn} which bounds the errors $R_n$ for both $f_n$ and $g_n$.
Suppose $X\xrightarrow{f, \eta} Y$, then if $f$ is smooth or continuously differentiable, $R_n(f_n)\to 0$, and
in fact we can obtain finite rates of convergence for $R_n(f_n)$,
thus yielding advice on setting $\sigma$. The second part of the Lemma corresponds to this situation.

However, as mentioned earlier in the paper introduction, a smooth $f$
does not ensure that $g(y)\triangleq g(X|y)$ is smooth or even continuous,
so we do not have rates for $R_n(g_n)$. We can nonetheless show that $R_n(g_n)$ would
generally converge to $0$, which is sufficient for there to be proper settings for $\sigma$
(i.e. $\sigma$ larger than the error, but also tending to $0$).

We note that the r.v.'s $X$ and $Y$ are interchangeable in this lemma 
since it does not assume $X\to Y$. The proof is given in the supplemental appendix.

\begin{lemma}
\label{lem:Rn}
Let $f_n$ be defined as in (\ref{eq:boxkernel}). Let $f(x)\triangleq \expec[{Y|x}]$.
Suppose ({{\rm{i}}}) $\expec Y^2< \infty$ and that $f$ is bounded;  $h\to 0$ and $nh\to \infty$. Then
$\expec{\frac{1}{n}\sum_1^n \abs{f_n(X_i)-f(X_i)}}\xrightarrow{n\to 0}0.$ 

 Suppose further ({{\rm{ii}}}) that $P_X$ has bounded support and that $f$ is continuously differentiable; $h=c_1n^{-\alpha}$ for
 some $c_1>0$ and $0<\alpha<1$. 
 
 Then we have
$\expec{\frac{1}{n}\sum_1^n \abs{f_n(X_i)-f(X_i)}}\leq c_2 n^{-\beta},$ 
for $\beta\triangleq \min\braces{(1-\alpha)/2, \alpha}$.
\end{lemma}


We can now prove the theorem of this section.

\subsection*{Proof of Theorem \ref{theo:coupled}}
Let $\bar\epsilon_{Y, i} \triangleq Y_i - f(X_i)$ and ${\bar\epsilon}_{X, i} \triangleq X_i - g(Y_i)$.
Note that, under our conditions on $\sigma$ both $H\paren{p_{n, \bar\epsilon_Y}}$ and
$H\paren{p_{n, \bar\epsilon_X}}$ are respectively
consistent estimators of $H\paren{\eta_{Y, f}}\triangleq H(\eta)$ and $H\paren{\eta_{X, g}}$
(see e.g. \cite{beirlant1997nonparametric}).
For any two densities $p, p'$ we write $\abs{p-p'}$ to
denote $\sup_t\abs{p(t) - p'(t)}$.

Given the assumption that $K$ has bounded derivative on $\real$, there exists a constant $c_K$ such that
\begin{align*}
\abs{p_{n, \bar\epsilon_Y} - p_{n, \epsilon_{Y}}}\leq \frac{c_K}{\sigma^2}\cdot\paren{\frac{1}{n}\sum_{i=1}^n \abs{f_n(X_i) - f(X_i)}}
\end{align*}
which implies
\begin{align*}
\expec\abs{p_{n, \bar\epsilon_Y} - p_{n, \epsilon_{Y}}}^{1/2}\leq \frac{\sqrt{c_KR_n(f_n)}}{\sigma},
\end{align*}
and also
\begin{align*}
\abs{p_{n, \bar\epsilon_X} - p_{n, \epsilon_{X}}}\leq \frac{c_K}{\sigma^2}\cdot\paren{\frac{1}{n}\sum_{i=1}^n \abs{g_n(Y_i) - g(Y_i)}}
\end{align*}
which also implies
\begin{align*}
\expec\abs{p_{n, \bar\epsilon_X} - p_{n, \epsilon_{X}}}^{1/2}\leq \frac{\sqrt{c_K R_n(g_n)}}{\sigma}.
\end{align*}

Thus by Lemma \ref{lem:Rn}, we have $\expec\abs{p_{n, \bar\epsilon_X} - p_{n, \epsilon_{X}}}^{1/2}\to 0$, which in turn implies by Markov's
inequality that $\abs{p_{n, \bar\epsilon_X} - p_{n, \epsilon_{X}}}\xrightarrow{P}0$.
Now since $P_X$ has bounded support, both $p_{n, \epsilon_X}$ and $p_{n, \bar\epsilon_X}$ have bounded support, and hence by Lemma
\ref{lem:entropy} we have $\abs{H\paren{p_{\epsilon_{X, g_n}}} - H\paren{ p_{\bar\epsilon_{X, g}} }}\xrightarrow{P}0$. Hence we also have
$\abs{H_n\paren{\eta_{X, g_n}} - H(\eta_{X, g})}\xrightarrow{P}0$.

Again by Lemma \ref{lem:Rn}, we have that, for $n$ sufficiently large, 
$\expec\abs{p_{n, \bar\epsilon_Y} - p_{n, \epsilon_{Y}}}^{1/2}\leq Cn^{-\beta/2}$ for some $C>0$.
Therefore by Markov's inequality, we have
$\pr{\abs{p_{n, \bar\epsilon_Y} - p_{ n, \epsilon_{Y}} }\geq \sqrt{C}n^{-\beta/4}}\to 0$. Now, under the exponential tail assumption
on the noise, all $Y_i$ samples are contained in a region of size $C'\log n$ with probability at least $1/n$. Thus, since $K$
is supported in $[-1, 1]$, both $p_{n, \bar\epsilon_X}$ and $p_{ n, \epsilon_{Y}}$ are $0$ outside a region of size $C''\log n$.
Let $T$ be as in Lemma \ref{lem:entropy}; for all $n$ sufficiently large, $\sqrt{C}n^{-\beta/4}< 1/(C''\log n)^2 = 1/T^2$. It follows by
Lemma \ref{lem:entropy} that $\abs{H\paren{p_{\epsilon_{Y, f_n}}} - H\paren{ p_{\bar\epsilon_{Y, f}} }}\xrightarrow{P}0$, and hence
that $\abs{H_n\paren{\eta_{Y, f_n}} - H(\eta_{Y, f})}\xrightarrow{P}0$.

The rest of the proof is similar to that of Theorem \ref{theo:relating_Rp_H} by calling on Lemma \ref{lem:complexity} and using the
consistency of $H_n(X)$ and $H_n(Y)$.
\qed

\section{Final Remarks}
We derived the first consistency results for an existing family of procedures for causal inference under the Additive Noise Model.
We obtained mild algorithmic requirements, and various distributional tail conditions which guarantee consistency. 
The present work focuses on the case of two r.v.s $X$ and $Y$, which captures the inherent difficulties of consistency. 
We believe however that the insights developed should extend to the case of random vectors under corresponding tail conditions. 
The details however are left for future work.

Another interesting multivariate situation is that of a causal network of r.v.s.
as in \citet{peters2011identifiability} dicussed earlier. 
  Extending our consistency results to this particular
multivariate case would primarily consist of extending our distributional tail conditions to the tails of 
distributions resulting from conditioning on appropriate sets of variables in the network. 
This is however a non-trivial extension as it involves, e.g. for the convergence of conditional entropies, 
some additional integration steps that have to be carefully worked out.
 
A possible future direction of investigation is to understand under what conditions finite sample rates can be
obtained for such procedures. For reasons explained earlier, we do not believe that this is possible 
without less general distributional assumptions.


\bibliography{refs}
\bibliographystyle{icml2014}

\newpage
\appendix

\section{Omitted figures from Section \ref{sec:inference}}
Some addtional experimental results were omitted in the main paper for space, and
are given in Figure \ref{fig:results2}.
\begin{figure*}[t]
\begin{center}
\subfigure[]
  {\includegraphics[height=1.5in]{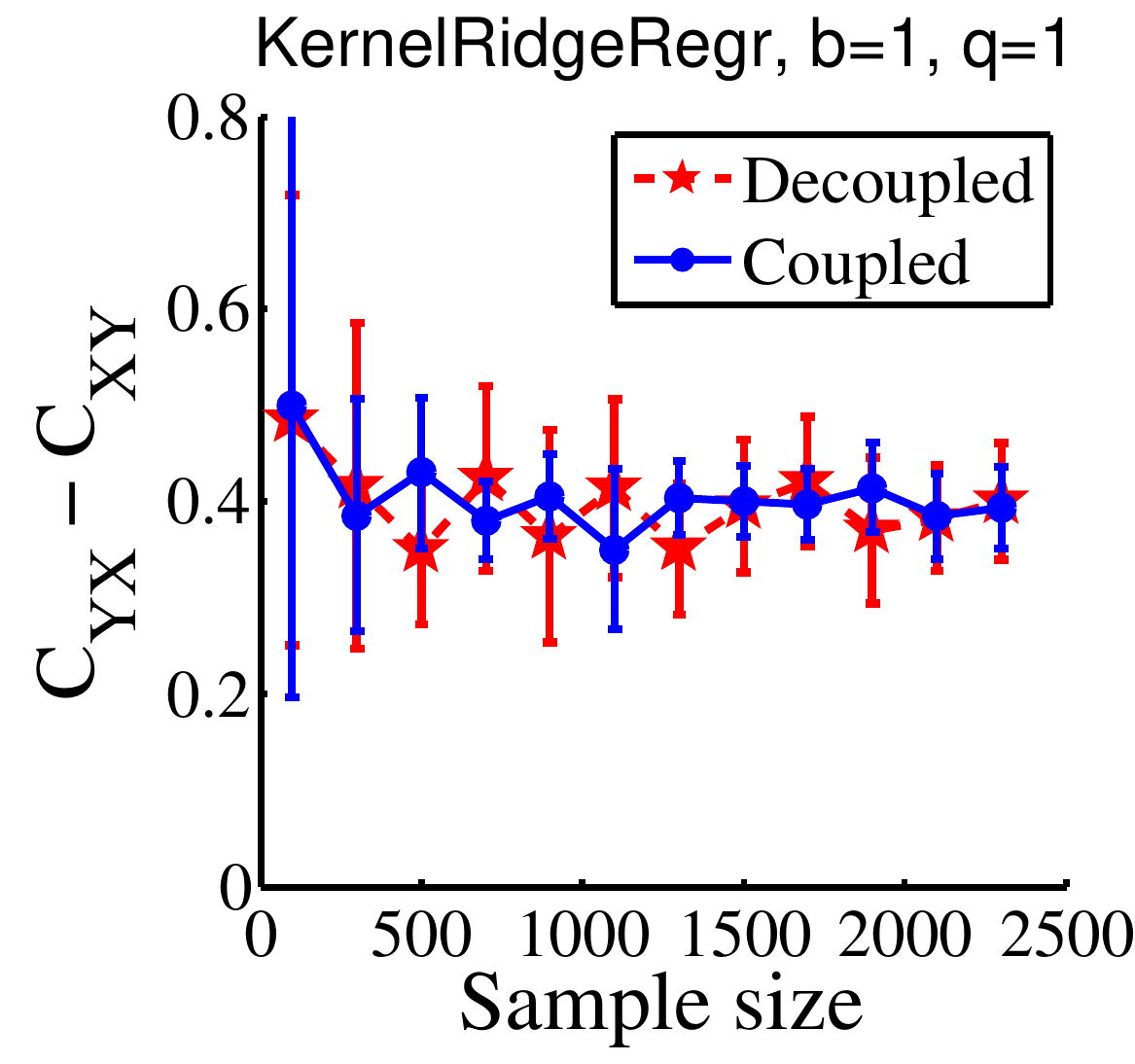}
  \label{fig:N_KerRidgeRegr}}
\subfigure[]
	{\includegraphics[height=1.5in]{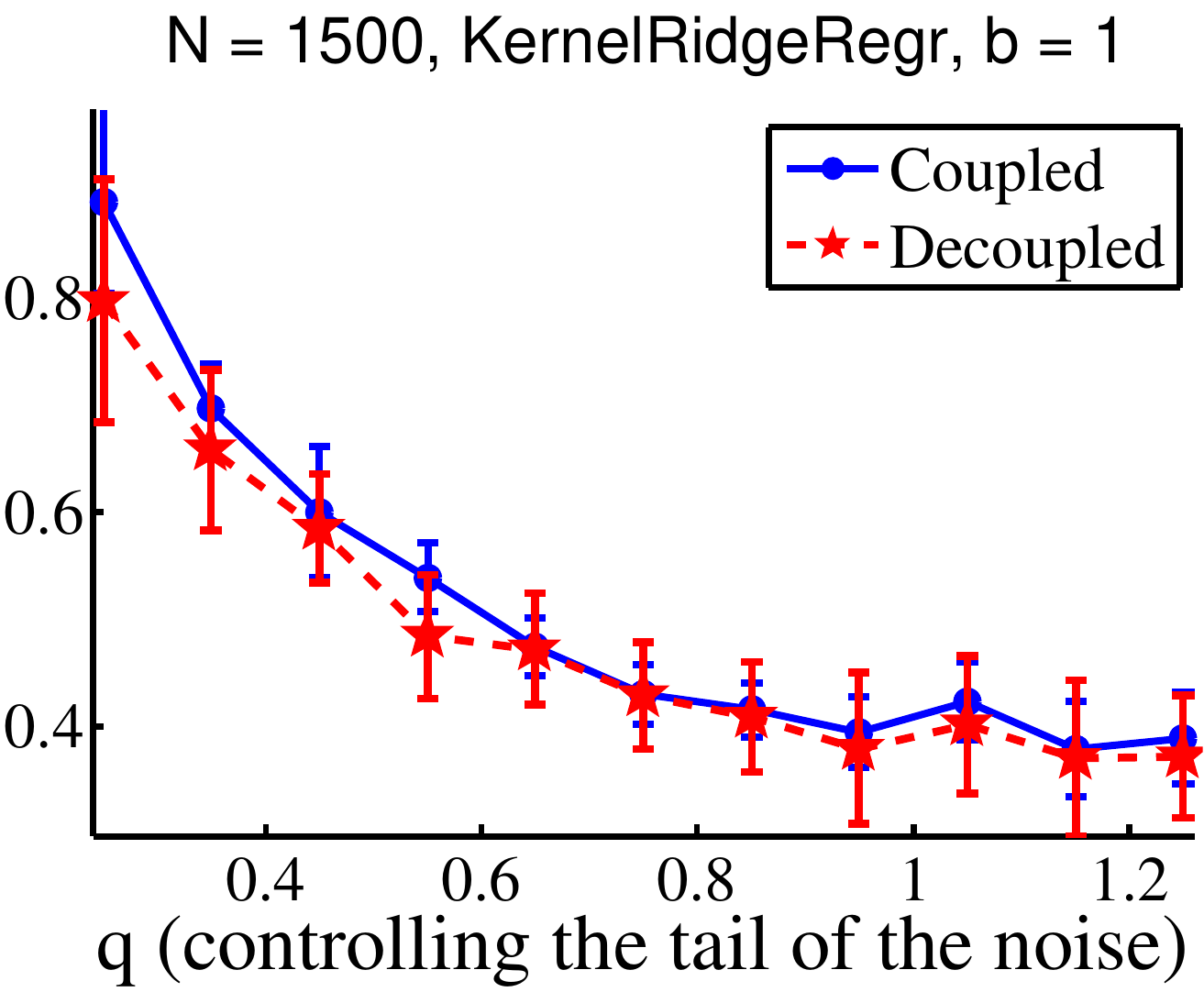}
  \label{fig:Q_KerRidgeRegr}}
\end{center}
\caption[font=small,labelsep=none] { \label{fig:results2}
Plot (a): same experiment as Fig. \ref{fig:N_KerRegr} but using KRR.
Plot (b): same experiment as Fig. \ref{fig:Q_KerRegr} but using KRR.
For properly tuned parameters, the selection of regression method does not seem to matter for the causal inference results.}
\end{figure*}

\section{Omitted Proofs: Section \ref{sec:decoupled}}
\begin{proof}[Proof of Lemma \ref{lem:propInduced}]
Note that, by assumption, both $p_X$ and $p_\eta$ are bounded.
 For any $x, y\in \real$, we have
 \begin{align}
  p_{X, Y}(x, y) = p_X(x)\cdot p_{Y|x}(y)=p_X(x)\cdot p_\eta(y - f(x)). \label{eq:pxy}
 \end{align}
 therefore $\frac{d}{dx}p_{X, Y}(x, y)$ is given by
 $$\frac{d}{dx}p_X(x)\cdot p_\eta(y - f(x)) - \frac{d}{dx}f(x)\cdot \frac{d}{dx}p_\eta(y - f(x)).$$

It is clear that $\sup_{x, y}\frac{d}{dx}p_{X, Y}(x, y)<\infty$. Similarly
$\sup_{x, y}\frac{d}{dy}p_{X, Y}(x, y)<\infty$. Also, since $p_\eta$ is bounded, we have from (\ref{eq:pxy}) that
for $\abs{t}$ sufficiently large, $p_{X, Y}(t, y)=0$ independent of $y$. Also, since $f$ and $p_X$ are bounded,
we have, independent of $x$, that for $\abs{t}$ sufficiently large, $p_{X, Y}(x, t)<C't^{-\alpha}$ for some $C'>0$.

Next the density of the residual $\eta_{Y, f'}$ of a function $f'$ is easily obtained as follows for any $t\in \real$.
\begin{align}
 p_{\eta_{Y, f'}}(t) &=\int_\real p_X(x)\cdot p_{Y|x}(t+f'(x))\, dx \nonumber\\
 &= \int_\real p_{X, Y}(x, t+f'(x))\, dx \label{eq:densityofresidual}\\
 &= \int_\real p_X(x)\cdot p_\eta(t+f'(x) - f(x))\, dx. \nonumber
\end{align}
Thus if $f'$ is bounded, there exists $T'>0$ such that for $\abs{t}>T'$,
$$p_{\eta_{Y, f'}}(t)\leq C' \abs{t}^{-\alpha}\cdot \int_\real p_X(x)\, dx = C' \abs{t}^{-\alpha},$$
where $\alpha$ is the same as for the bound on $p_\eta$ from the assumption.

We have similarly for any function $g$ that,
\begin{align}
 p_{\eta_{X, g}}(t) &=\int_{p_Y>0}p_Y(y)\cdot p_{X|y}(t+g(y))\, dy \nonumber\\
 &=\int_\real p_{X, Y}(t+g(y), y) \, dy. \label{eq:densityofresidual2}\\
 &= \int_\real p_X(t+g(y))\cdot p_\eta(y - f(t+ g(y)))\, dy.\nonumber
\end{align}
If $g$ is bounded, then for $t$ sufficiently large, $p_X(t+g(y))=0$ for all $y$ so $p_{\eta_{X, g}}(t)=0$.
\end{proof}

\section{Omitted Proofs: Section \ref{sec:coupled}}
We denote the random $n$-samples as $X^n\triangleq \braces{X_i}_1^n$ and $Y^n\triangleq\braces{Y_i}_1^n$,
throughout this section.

We bound $R_n(f_n)$ and $R_n(g_n)$ in Lemma \ref{lem:Rn} which makes use of Lemma \ref{lem:book}.
We note that the r.v.'s $X$ and $Y$ are interchangeable in the two lemmas \ref{lem:Rn} and \ref{lem:book},
since they do not assume $X\to Y$.

\begin{lemma}[\cite{GKKW:81}]
 \label{lem:book}
 For any positive $f:\real \to \real$ such that $\expec f < \infty$, there exists $c_0$ such that
  $\expec_{X, X^n} \braces{\frac{1}{n_{X, h}}\sum_{X_i: \abs{X_i-X}<h} f(X_i)} \leq c_0 \expec f(X).$
\end{lemma}

%
\begin{proof}[Proof of Lemma \ref{lem:Rn}]
We simply have to show that $\expec{\frac{1}{n}\sum_1^n \abs{f_n(X_i)-f(X_i)}^2}\xrightarrow{n\to \infty}0$ since
by H\"{o}lder's inequality and Jensen's inequalities, for any $\phi_n(\cdot)$,
\begin{align*}
 \expec{\frac{1}{n}\sum_i \abs{\phi_n(X_i)}} \leq& \expec{\sqrt{\frac{1}{n}\sum_i \abs{\phi_n(X_i)}^2}}\\
\leq& \sqrt{\expec{\frac{1}{n}\sum_i \abs{\phi_n(X_i)}^2}}.
\end{align*}

For assumption ({\rm{i}}), pick any $\epsilon>0$. We will show that for $n$ sufficiently large, the above expectation is at most $(7+ 3c_0)\epsilon$,
where $c_0$ is as in Lemma \ref{lem:book}. The further claim of assumption ({\rm{ii}}) will be obtained along the way.

First condition on $X^n$, fixing $x =X_i$ for some $X_i$, and taking expectation
with respect to the randomness in
$Y^n\triangleq\braces{Y_i}_1^n$.
We have by a standard bias-variance decomposition (see e.g. \cite{GKKW:81}) that
$\Expectation_{Y^n|X^n}\abs{f_n(x) - f(x)}^2$
\begin{align}
 \leq \frac{C}{n_{x, h}} + \abs{\frac{1}{n_{x, h}} \sum_{\abs{X_j-x}<h} f(X_j) - f(x) }^2
 = A_x + B_x, \label{eq:biasvar}
 \end{align}
 for some $C$ depending on the variance of $Y$.

We start with a bound on the first term of (\ref{eq:biasvar}). Pick an interval $S$ such that $P_X(\real\setminus S)<\epsilon$.

Consider an $(h/2)$-cover $Z$ of $S$ such that for every $z\in S$, the interval $[z-h/2, z+h/2]$ is contained in $S$.
We can pick such a $Z$ of size at most $2\Sigma(S)/h$. Note that for any $x\in [z-h/2, z+h/2]$,
$n_{x, h}\geq n_{z, h/2} \triangleq \abs{\braces{X_i: \abs{z-X_i}<h/2}}$.
We then have
\begin{align*}
\frac{1}{n}\sum_{i=1}^n A_{X_i} &= \frac{1}{n}\sum_{i=1}^n \frac{C}{n_{X_i, h}}\paren{\ind{X_i\in S} + \ind{X_i\notin S}}\\
&\leq  \frac{1}{n}\sum_{j=1}^n \frac{C}{n_{X_i, h}}\ind{X_i\in S} + \frac{1}{n}\sum_{i=1}^n \ind{X_i\notin S}\\
&\leq \frac{1}{n}\sum_{z\in Z}\sum_{X_i: \abs{z-X_i}\leq h/2}\frac{C}{n_{X_i, h}} + \frac{1}{n}\sum_{i=1}^n \ind{X_i\notin S}\\
&\leq \frac{1}{n}\sum_{z\in Z}\frac{C\cdot n_{z, h/2}}{n_{z, h/2}} + \frac{1}{n}\sum_{i=1}^n \ind{X_i\notin S}.
\end{align*}
Therefore by taking expectation over $X^n$ and letting $nh$ sufficiently large, we have
$$\expec \frac{1}{n}\sum_{i=1}^n A_{X_i} \leq \frac{2C\cdot\Sigma (S)}{n h} + P_X(\real\setminus S)\leq 2\epsilon.$$
Under assumption ({\rm{ii}}), pick $S$ larger than the support of $P_X$,
we have by the same equation above that for large $n$
$$\expec \frac{1}{n}\sum_{i=1}^n A_{X_i} \leq \frac{2C\cdot\Sigma (S)}{n h} = \frac{2C\cdot\Sigma (S)}{c_1n^{1-\alpha}}.$$

 We now turn to the second term of (\ref{eq:biasvar}). Under assumption ({\rm{ii}}) the function $f$ is Lipschitz continuous
 and we therefore have for some constant $c_f$ that
 $A_x \leq c_f h^2 = c_f c_1n^{-2\alpha}$. Combining with the bound on $A_x$ gives the result for assumption ({\rm{i}}).

 For assumption ({\rm{i}}) we proceed as follows. It is well known that bounded
 uniformly continuous functions are dense in $L_{2, P_X}$ for any $P_X$.
 Therefore let $\ft$ be a bounded uniformly continuous function such that $\norm{\ft - f}_{2, P_X}<\sqrt{\epsilon}$. Since $h=h(n)\to 0$, we have $\sup_{\abs{x, x'}<h}\abs{\ft(x) - \ft(x')}^2 < \epsilon$ for $n$ sufficiently
 large.
 The second term of the r.h.s. of the above equation (\ref{eq:biasvar}) can then be bounded as follows. If $n_{x, h}=1$, then
 $B_x = 0$. Otherwise, if $n_{x, h}>1$, we have
 \begin{align*}
B_x \leq& \frac{3}{n_{x, h}} \sum_{\abs{X_j-x}<h}
\left( \abs{f(X_j) - \ft(X_j)}^2 \right. \\
&\left. +  \abs{\ft(X_j) - \ft(x)}^2 + \abs{\ft(x) - f(x)}^2\right) \\
\leq& 3\epsilon +\abs{\ft(x) - f(x)} \\
&+ \frac{3}{n_{x, h}} \sum_{\abs{X_j-x}<h}\abs{f(X_j) - \ft(X_j)}^2\\
\leq& 3\epsilon + \paren{\frac{3}{n_{x, h}} + 1}\abs{\ft(x) - f(x)} \\
&+ \frac{3}{n_{x, h} -1} \sum_{\abs{X_j-x}<h, X_j\neq x}\abs{f(X_j) - \ft(X_j)}^2.
\end{align*}

Therefore taking expectation over $X^n$, and applying Lemma \ref{lem:book} to the second term above for $x=X_i$,
we have
\begin{align*}
 \expec B_{X_i} \leq& 3\epsilon + 2\expec_X\abs{\ft(X) - f(X)}^2 \\
 &+ 3c_0\expec_X\abs{\ft(X) - f(X)}^2 = (5+3c_0)\epsilon,
\end{align*}
so that $\expec \frac{1}{n}\sum_i B_{X_i} \leq (5+3c_0)\epsilon$.
\end{proof}

\end{document}